\documentclass[10pt,a4paper]{article}
\usepackage{a4wide}

\usepackage{hyperref}
\usepackage{url}

\usepackage[utf8]{inputenc} % allow utf-8 input
\usepackage[T1]{fontenc}    % use 8-bit T1 fonts
\usepackage{booktabs}       % professional-quality tables
\usepackage{amsfonts}       % blackboard math symbols
\usepackage{nicefrac}       % compact symbols for 1/2, etc.
\usepackage{microtype}      % microtypography

\usepackage{graphicx}
\usepackage{amsmath}
\usepackage{amssymb}
\usepackage{tabularx}
\usepackage{tikz,environ,pgfplots,booktabs}
\usepackage{amsthm}

\usepackage{soul}

\newcommand{\ip}[2]{\left\langle#1,#2\right\rangle}

\newcommand{\abs}[1]{\left|#1\right|}
\newcommand{\norm}[1]{\left\|#1\right\|}

\def\cC{\mathcal{C}}
\def\CC{\mathbb{C}}
\def\NN{\mathbb{N}}

\def\RR{\mathbb{R}}

\def\e{\epsilon}

\def\l{\lambda}

\def\s{\sigma}

\def\DD{\boldsymbol{\Delta}}

\def\cF{\mathcal{F}}

\newtheorem{theorem}{Theorem}
\newtheorem{definition}[theorem]{Definition}

\newtheorem{lemma}[theorem]{Lemma}

\begin{document}

\title{On the Transferability of Spectral Graph Filters}

\author{Ron Levie\\
        \texttt{\small levie@math.tu-berlin.de}
				\and
        Elvin Isufi\\
				\texttt{\small E.Isufi-1@tudelft.nl}
				\and
        Gitta Kutyniok\\
			 \texttt{\small kutyniok@math.tu-berlin.de}
			}
			
			\date{}

\maketitle

\begin{abstract}
This paper focuses on spectral filters on graphs, namely filters defined as elementwise multiplication in the frequency domain of a graph. In many graph signal processing settings, it is important to transfer a filter from one graph to another. One example is in graph convolutional neural networks (ConvNets), where the dataset consists of signals defined on many different graphs, and the learned filters should generalize to signals on new graphs, not present in the training set. A necessary condition for transferability (the ability to transfer filters) is stability. Namely, given a graph filter, if we add a small perturbation to the graph, then the filter on the perturbed graph is a small perturbation of the original filter. It is a common misconception that spectral filters are not stable, and this paper aims at debunking this mistake. We introduce  a space of filters, called the Cayley smoothness space, that contains the filters of state-of-the-art spectral filtering methods, and whose filters can approximate any generic spectral filter. For filters in this space, the perturbation in the filter is bounded by a constant times the perturbation in the graph, and filters in the Cayley smoothness space are thus termed linearly stable. By combining stability with the known property of equivariance, we prove that graph spectral filters are transferable.
\end{abstract}

\section{Introduction}

The success of convolutional neural networks (ConvNets) on Euclidean domains ignited an interest in recent years in extending these methods to graph structured data. A graph ConvNet is a mapping, that receives a signal defined over the vertices of a graph, and returns a value in some output space. A graph ConvNet consists of many layers of computations, where each layer computes a set of filters of the output of the previous layer, followed by a pointwise nonlinearity, and optionally a pooling step and non-convolution layers. % (see \cite{review_new} for a review).
 In a machine learning setting, the general architecture of the ConvNet is fixed, but the specific filters to use in each layer are free parameters. In training, the filter coefficients are optimized to minimize some loss function. In some situations, both the graph and the signal defined on the graph are variables in the input space of the ConvNet. In these situations, if two graphs represent the same underlying phenomenon, and the two signals given on the two graphs are similar in some sense, the output of the ConvNet on both signals should be similar as well. This property is typically termed transferability, and is an essential requirement if we wish the ConvNet to generalize well on the test set.
Analyzing and proving transferability is the subject of this paper.

A necessary condition of any reasonable definition of transferability is stability. Namely, given a filter, if the topology of a graph is perturbed, then the filter on the perturbed graph is close to the filter on the un-perturbed graph. Without stability it is not even possible to transfer a filter from a graph to another very close graph, and thus stability is necessary for transferability. 
Previous work studied the behavior of graph filters with respect variations in the graph. \cite{segarra2017optimal} provided numerical results on the robustness of polynomial graph filters to additive Gaussian perturbations of the eigenvectors of the graph Laplacian. Since the eigendecomposition is not stable to perturbations in the topology of the graph, this result does not prove robustness to such perturbations. \cite{isufi2017filtering}  showed that the expected graph filter under random edge losses is equal to the accurate output. However, \cite{isufi2017filtering} did not bound the error in the output in terms of the error in the graph topology. In this paper we show the linear stability of graph filters to general perturbations in the topology.

There are generally two approaches to defining convolution on graphs, both generalizing the standard convolution on Euclidean domains \cite{review_new,wu2019comprehensive}. Spatial approaches generalize the idea of a sliding window to graphs. Here, the main challenge is to define a way to translate a filter kernel along the vertices of the graph. Some popular examples of spatial methods are \cite{GNN_1,GNN_2,Monti2017GeometricDL}. Spectral methods are inspired by the convolution theorem in Euclidean domains, that states that convolution in the spatial domain is equivalent to pointwise multiplication in the frequency domain. The challenge here is to define the frequency domain and the Fourier transform of graphs. The basic idea is to define the graph Laplacian, or some other graph operator that we interpreted as a shift operator, and to use its eigenvalues as frequencies and its eigenvectors as the corresponding pure harmonics \cite{ortega2018graph}. Decomposing a graph signal to its pure harmonic coefficients is by definition the graph Fourier transform, and filters are defined by multiplying the different frequency components by different values. For some examples of spectral methods see, e.g., 
\cite{bruna2013spectral,defferrard2016convolutional,Cayley1,gama2018convolutional}.
Additional references for both methods can be found in \cite{wu2019comprehensive}.

The great majority of researchers from the graph ConvNet community currently focus on developing spatial methods. One typical motivation for favoring spatial methods is the claim that spectral methods are not transferable, and thus do not generalize well on graphs unseen in the training set. The goal in this paper is to debunk this misconception, and to show that state-of-the-art spectral graph filtering methods are transferable.
This paper does not argue against spatial methods, but shows the potential of spectral approaches to cope with datasets having varying graphs. We would like to encourage researches to reconsider spectral methods in such situations. 
Interestingly, \cite{Renjie_near_state} obtained near state-of-the-art results in face recognition using spectral filters on variable graphs, without any modification to compensate for the ``non-transferability''.

\section{Preliminaries}
\subsection{Spectral graph filters}

Consider an undirected weighted graph ${\cal G}=\{E,V,{\bf W}\}$, with vertices $V=\{1,\ldots,N\}$, edges $E$, and adjacency matrix ${\bf W}$. The adjacency matrix ${\bf W}$ is symmetric and represents the weights of the edges, where $w_{n,m}$ is nonzero only if vertex $n$ is connected to vertex $m$. Consider the degree matrix ${\bf D}$, defined as the diagonal matrix with entries $d_{n,n}=\sum_{m=1}^N w_{n,m}$.

The frequency domain of a graph is determined by choosing a shift operator, namely a self-adjoint operator $\DD$ that respects the connectivity of the graph. 
As a prototypical example, we consider the unnormalized Laplacian $\DD ={\bf D}-{\bf W}$, which depends linearly on ${\bf W}$.
Other examples of common shift operators are the normalized Laplacian $\DD_{\rm n}={\bf I}-{\bf D}^{-1/2}{\bf W}{\bf D}^{-1/2}$, and the adjacency matrix itself. In this paper we call a generic self-adjoint shift operator \emph{Laplacian}, and denote it by $\DD$.
Denote the eigenvalues of $\DD$ by $\{\l_n\}_{n=1}^N$, and the eigenvectors by $\{\phi_n:V\rightarrow\CC\}_{n=1}^N$. The Fourier transform of a graph signal $f:V\rightarrow\CC$ is given by the vector of frequency intensities
\[\cF f = (\ip{f}{\phi_n})_{n=1}^N,\]
where $\ip{u}{v}$ is an inner product in $\CC^N$, e.g., the standard dot product.
The inverse Fourier transform of the vector $(v_n)_{n=1}^N$ is given by
\[\cF^* (v_n)_{n=1}^N = \sum_{n=1}^Nv_n\phi_n.\]
Since $\{\phi_n\}_{n=1}^N$ is an orthonormal basis, $\cF^*$ is the inverse of $\cF$. 
A spectral graph filter ${\bf G}$ based on the coefficients $(g_n)_{n=1}^N$ is defined by
\begin{equation}
{\bf G} f = \sum_{n=1}^Ng_n\ip{f}{\phi_n}\phi_n.
\label{filt0}
\end{equation}
Any spectral filter defined by (\ref{filt0}) is \emph{equivariant}, namely, does not depend on the indexing of the vertices. Re-indexing the vertices in the input, results in the same re-indexing of vertices in the output. 

Spectral filters defined by (\ref{filt0}) have two disadvantages. First, as shown in Subsection \ref{The misconception of non-transferability of spectral graph filters}, they are not transferable. Second, they entail high computational complexity. Formula (\ref{filt0}) requires the computation of the eigendecomposition of the Laplacin $\DD$, which is computationally demanding and can be unstable when the number of vertices $N$ is large. Moreover, there is no general ``graph FFT'' algorithm for computing the Fourier transform of a signal $f\in L^2(V)$, and (\ref{filt0}) requires computing the frequency components $\ip{f}{\phi_n}$ and their summation directly. 

To overcome these two limitations, state-of-the-art methods, like \cite{defferrard2016convolutional,art:ARMA,Cayley1,gama2018convolutional},  are based on \emph{functional calculus}. Functional calculus is the theory of applying functions $g:\CC\rightarrow\CC$ on normal operators in Hilbert spaces. In the special case of a self-adjoint or unitary operator $\mathbf{T}$ with a discrete spectrum, $g(\mathbf{T})$ is defined by
\begin{equation}
g(\mathbf{T})f=\sum_{n} g(\l_n)\ip{f}{\phi_n}\phi_n,
\label{eq:FC1}
\end{equation}
for any vector $f$ in the Hilbert space, where $\{\l_n,\phi_n\}$ is the eigendecomposition of the operator $\mathbf{T}$. The operator $g(\mathbf{T})$ is normal for general $g:\CC\rightarrow\CC$, self-adjoint for $g:\CC\rightarrow\RR$, and unitary for $g:\CC\rightarrow e^{i\RR}$ (where $e^{i\RR}$ is the unit complex circle). 

Definition (\ref{eq:FC1}) is canonical in the following sense. In the special case where 
\[g(\l)=\frac{\sum_{l=0}^{L}c_l\l^l}{\sum_{l=0}^{L}d_l\l^l}\]
is a rational function, $g(\mathbf{T})$ can be defined in two ways. First, by (\ref{eq:FC1}), and second by compositions, linear combinations, and inversions, as 
\begin{equation}
g(\mathbf{T}) = \Big(\sum_{l=0}^{L}c_l\mathbf{T}^l\Big)\Big(\sum_{l=0}^{L}d_l\mathbf{T}^l\Big)^{-1}
\label{eq:FC_ratio}
\end{equation}
It can be shown that (\ref{eq:FC1}) and (\ref{eq:FC_ratio}) are equivalent. Moreover, definition (\ref{eq:FC1}) is also canonical in regard to non-rational functions. Loosly speaking, if a rational function $q$ approximates the function $g$, then the operator $q(\mathbf{T})$ approximates the operator $g(\mathbf{T})$. 
 
Implementation (\ref{eq:FC_ratio}) overcomes the limitation of definition (\ref{filt0}), where now filters are defined via (\ref{eq:FC1}) with polynomial or rational function $g$. By relying on the spatial operations of compositions, linear combinations, and inversions, the computation of a spectral filter is carried out entirely in the spatial domain, without ever resorting to spectral computations. Thus, no eigendecomposition and Fourier transforms are ever computed. The inversions in $g(\mathbf{T})f$ involve solving systems of linear equations, which can be computed directly if $N$ is small, or by some iterative approximation method for large $N$.
Methods like \cite{defferrard2016convolutional,Kipf2017,ortega2018graph,gama2018convolutional} use polynomial filters, and \cite{art:ARMA,Cayley1,ARMA_ConvNet} use rational function filters. We term spectral methods based on functional calculus \emph{functional calculus filters}.

\subsection{The misconception of non-transferability of spectral graph filters}
\label{The misconception of non-transferability of spectral graph filters}

The non-transferability claim is formulated based on the sensitivity of the Laplacian eigendecomposition to small perturbation in $\mathbf{W}$, or equivalently in $\DD$. Namely, a small perturbation of $\DD$ can result in a large perturbation of the eigendecomposition $\{\l_n,\phi_n\}_{n=1}^N$, which results in a large change in the filter defined via (\ref{filt0}).
This argument, while true, does not prove non-transferability, since state-of-the-art spectral methods do not explicitly use the eigenvectors, and do not parametrize the filter coefficients $g_n$ via the index $n$ of the eigenvalues. Instead, state-of-the-art methods are based on functional calculus, and define the filter coefficients using a function $g:\RR\rightarrow\CC$, as $g(\l_n)$. 
The parametrization of the filter coefficients by $g$ is indifferent to the specifics of how the spectrum is indexed, and instead represents an overall response in the frequency domain, where the \emph{value} of each frequency determines its response, and not its index.
In functional calculus filters defined by (\ref{eq:FC1}), a small perturbation of $\DD$ that results in a perturbation of $\l_n$, also results in a perturbation of the coefficients $g(\l_n)$. It turns out, as we prove in this paper, that the perturbation in $g(\l_n)$ implicitly compensates for the instability of the eigendecomposition, and functional calculus spectral filters are stable.

\section{Main results}
\subsection{Transferability of functional calculus filters}

In this paper, we define transferability as the linear robustness of the filter to re-indexing of the vertices and perturbation of the topology of the graph.
Thus, to formulate transferability, we combine equivariance with stability.  Since spectral filters are known to be equivariant, transferability is equivalent to stability. Thus, our goal is to prove stability.

\subsection{Linear stability of spectral filters}

Stability is proven on a dense subspace of filters is $L^p(\RR)$, which we term the Cayley smoothness space.
The definition of the Cayley smoothness space is based on the Cayley transform $\cC:\RR\rightarrow e^{i\RR}$, defined by $\cC(x)=\frac{x-i}{x+i}$. 
\begin{definition}
\label{def:Cay1}
The \emph{Cayley smoothness space} $Cay^1(\RR)$ is the subspace of functions $g\in L^2(\RR)$ of the form $g(\l)=q\big(\mathcal{C}(\l)\big)$, where $q:e^{i\RR}\rightarrow\CC$ is in $L^2(e^{i\RR})$, and has classical Fourier coefficients $\{c_l\}_{l=1}^{\infty}$ satisfying $\norm{g}_{\cC}:=\sum_{l=1}^{\infty} l\abs{c_l}<\infty$.
\end{definition}
The mapping $g\mapsto\norm{g}_{\cC}$ is a seminorm. It is not difficult to show that $Cay^1(\RR)$ is dense in each $L^p(\RR)$ space with $1\leq p<\infty$.
Intuitively, Cayley smoothness implies decay of the filter kernel in the spatial domain, since it models smoothness in a frequency domain. This can be formulated rigorously for graph filters based on Cayley polynomials ($g(\l)=q\big(\cC(\l)\big)$ with finite expansion $\{c_l\}_{l=1}^L$) \cite[Theorem 4]{Cayley1}.

For filters in the Cayley smoothness space we can obtain a linear rate of convergence, which is our main contribution.
\begin{theorem}
\label{theo:Cayley1}
Let $\DD\in \CC^{N\times N}$ be a self-adjoint matrix that we call Laplacian. Let $\DD'=\DD+\mathbf{E}$ be self-adjoint, such that $\norm{\mathbf{E}}< 1$.
Let $g\in Cay^1(\RR)$. Then 
\begin{equation}
\begin{split}
\norm{g(\DD)-g(\DD')} \leq  &\norm{g}_{\cC} \Big((\norm{\DD}+1)\frac{\norm{\mathbf{E}}}{1-\norm{\mathbf{E}}}  +\norm{\mathbf{E}} \Big) \\
                       =    &O(\norm{\DD-\DD'}).
\end{split}
\label{eq:Cayley1}
\end{equation}
\end{theorem}

\section{Examples}
\noindent
\emph{ChebNets.} Consider the normalized translated Laplacian $\DD_{\rm n}-\mathbf{I}$. In ChebNets \cite{defferrard2016convolutional}, $g$ is a polynomial, and since the spectrum of $\DD_{\rm n}-\mathbf{I}$ is in $[-1,1]$, the values of $g$ outside $[-1,1]$ do not affect the filter $g(\DD_{\rm n}-\mathbf{I})$. Thus, we may assume that $g$ is a polynomial in $[-1,1]$, and padded outside $[-1,1]$ to obtain a smooth compactly supported function. It is easy to see that such a $g$ is in $Cay^1(\RR)$. Thus, for two translated normalized Laplacians $\DD_{\rm n}-\mathbf{I}$ and $\DD_{\rm n}'-\mathbf{I}$ of two graphs, $\norm{g(\DD_{\rm n}-\mathbf{I})-g(\DD'_{\rm n-\mathbf{I}})} =O(\norm{\DD_{\rm n}-\DD'_{\rm n}})$. 
\newline
\emph{General rational functions.} The above claim is also true for general rational functions, if we assume that the spectrum of $\DD,\DD'$ is contained in some pre-defined band $[0,M]$. Thus, the polynomial filters of \cite{Kipf2017,gama2018convolutional} and the ARMA rational function filters of \cite{art:ARMA,ARMA_ConvNet} are also transferable, under the assumption of uniformly bounded Laplacians.
\newline
\emph{CayleyNets.} CayleyNets \cite{Cayley1} are always transferable, since a Cayley filter is by definition in $Cay^1(\RR)$, with finite expansion.

\begin{figure}
\begin{center}
\includegraphics[width=0.91\columnwidth]{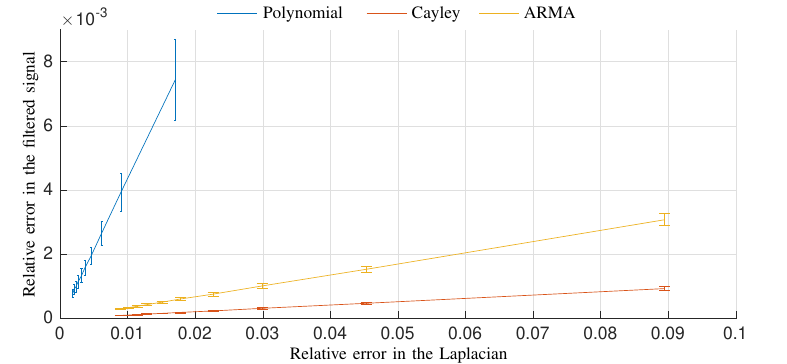}
\end{center}
\caption{Average relative error in the output of the filter as a function of the average relative error in the Laplacian, for three fixed filters: low-pass polynomial and Cayley filters of order 3, and all-pass ARMA filter of order 3.}
\label{fig:mainFig}
\end{figure}

To corroborate the proposed theoretical result, in Figure~\ref{fig:mainFig} we test the above three examples in the Molene weather dataset\footnote{Access to the raw data is possible from \url{https://donneespubliques.meteofrance.fr/donnees libres/Hackathon/RADOMEH.tar.gz}}. The graph comprises $N=32$ weather stations, with weights given as the Gaussian of the physical distances between stations.
Each of the $744$ graph signals is a temperature recording. For the polynomial filter we consider the normalized Laplacian, while for the Cayley and ARMA filters we consider the unnormalized Laplacian.  
The results are averaged over $100$ different perturbations in the topology and the $744$ graph signals. The experimental results concord with the theoretical linear stability property.

\section{Proof of Theorem \ref{theo:Cayley1}}

We start with a useful lemma.
\begin{lemma}
\label{Lamma_op_poly_con0}
Suppose $\mathbf{B},\mathbf{D},\mathbf{E}\in\CC^{N\times N}$ are self-adjoint matrices satisfying $\mathbf{B}=\mathbf{D}+\mathbf{E}$, and $\norm{\mathbf{B}},\norm{\mathbf{D}}\leq C$ for some $C>0$. 
Then for every $l\geq 0$
\begin{equation}
\norm{\mathbf{B}^l-\mathbf{D}^l} \leq l C^{l-1}\norm{\mathbf{E}}.
\label{eq:_op_poly_con0}
\end{equation}
\end{lemma}
\begin{proof}
Let $l\in\NN$. 
\begin{equation} \label{eeryqo1}
\mathbf{D}^l -\mathbf{B}^l =   \mathbf{D}^{l-1}\left(\mathbf{D}-\mathbf{B}\right)  +\left(\mathbf{D}^{l-1} -\mathbf{B}^{l-1} \right)\mathbf{B}
\end{equation}
so
\begin{equation} \label{eeryqo18}
\norm{\mathbf{D}^l -\mathbf{B}^l} \leq  C^{l-1}\norm{\mathbf{E}}  +\norm{\mathbf{D}^{l-1} -\mathbf{B}^{l-1} }C.
\end{equation}
Now, (\ref{eq:_op_poly_con0}) follows by repeatedly using (\ref{eeryqo18}) with decreasing powers $l-j$, $j=1,\ldots,l-1$.
\end{proof}

Next, we cite a general property from spectral theory.
\begin{lemma}
\label{lemma_functional_calculus1}
Let $T$ be a bounded normal operator in a Hilbert space. Let $\s$ be the spectrum of $T$. Define the infinity norm on the space of bounded continuous functions
$f:\s\rightarrow \CC$ by
\[\norm{f-g}^{\s}_{\infty} = \sup_{x\in\s}\abs{f(x)-g(x))}.\]
Then
\[\norm{f(T)-g(T)} = \norm{f-g}^{\s}_{\infty}\]
where the norm in the left-hand-side is the operator norm.
\end{lemma}

To prove Theorem \ref{theo:Cayley1}, we start with a version the theorem restricted to $g=q\circ\mathcal{C}\in Cay^1(\RR)$ where $q$ has a finite expansion with coefficients $(c_l)_{l=1}^L$.
\begin{proof}[Proof of Theorem \ref{theo:Cayley1} for finite Cayley expansions]
Note that
\[\begin{split}
\mathcal{C}(\DD)-\mathcal{C}(\DD') = &(\DD-i)(\DD+i)^{-1} -(\DD-i)(\DD'+i)^{-1} \\
 +& (\DD-i)(\DD'+i)^{-1}- (\DD'-i)(\DD'+i)^{-1}
\end{split}\]
so
\[\norm{\mathcal{C}(\DD)-\mathcal{C}(\DD')}\]
\[\leq \norm{(\DD-i)}\norm{(\DD+i)^{-1}-(\DD'+i)^{-1}} + \norm{(\DD'+i)^{-1}}\norm{\mathbf{E}}.\]
By the fact that the spectrum of $\DD'$ is real, $\norm{(\DD'+i)^{-1}}\leq 1$, and we have
\begin{equation}
\begin{split}
 & \norm{\mathcal{C}(\DD)-\mathcal{C}(\DD')}  \\
& \leq (\norm{\DD}+1)\norm{(\DD+i)^{-1}-(\DD'+i)^{-1}}  + \norm{\mathbf{E}}.
\end{split}
\label{eq:333e3ee}
\end{equation}

Let us bound $\norm{(\DD+i)^{-1}-(\DD'+i)^{-1}}$ in terms of $\norm{\mathbf{E}}$. Since $\norm{\mathbf{E}}<1$ we may expand
\[\begin{split}
(\DD+i+\mathbf{E})^{-1} = &(\DD+i)^{-1}\big(I+\mathbf{E}(\DD+i)^{-1}\big)^{-1}\\
               = &(\DD+i)^{-1}\Big(\sum_{k=0}^{\infty}(-1)^k(\mathbf{E}(\DD+i)^{-1})^k \Big)
\end{split}\]
so, by $\norm{(\DD+i)^{-1}}\leq 1$,
\begin{equation}
\norm{(\DD+i)^{-1} - (\DD'+i)^{-1}}  \leq \frac{\norm{\mathbf{E}}}{1-\norm{\mathbf{E}}}.
\label{eq:dfgdrg5lld}
\end{equation}
Now, by (\ref{eq:333e3ee}) and (\ref{eq:dfgdrg5lld}),
\begin{equation}
\norm{\mathcal{C}(\DD)-\mathcal{C}(\DD')} \leq (\norm{\DD}+1)\frac{\norm{\mathbf{E}}}{1-\norm{\mathbf{E}}} + \norm{\mathbf{E}}.
\label{eq:Cay_proof_b1}
\end{equation}
Observe that $\mathcal{C}(\DD)$ and $\mathcal{C}(\DD')$ are unitary, so their spectrum is bounded by $C=1$. Thus, by Lemma \ref{Lamma_op_poly_con0} and the triangle inequality on the polynomial expansion of $p\big(\mathcal{C}(\DD)\big)-p\big(\mathcal{C}(\DD')\big)$,
\[
\begin{split}
\norm{g(\DD)-g(\DD')}  =& \norm{p\big(\mathcal{C}(\DD)\big)-p\big(\mathcal{C}(\DD')\big)}\\
                    \leq& \sum_{l=1}^L l c_l \norm{\mathcal{C}(\DD)-\mathcal{C}(\DD')}
\end{split}
\]
which gives (\ref{eq:Cayley1}).
\end{proof}

\begin{proof}[Proof of Theorem \ref{theo:Cayley1}]
Theorem \ref{theo:Cayley1} follows the above result by a simple density argument. Given $g=q\circ\mathcal{C}\in Cay^1(\RR)$, we consider the truncations $g_L=q_L\circ\mathcal{C}$ , where $q_L$ is restricted to the coefficients $(c_l)_{l=1}^L$.
We base a three-epsilon argument on the expansion, for any $L\in\NN$, 
\begin{equation}
\begin{split}
\norm{g(\DD)-g(\DD')} \leq &\norm{g(\DD)-g_L(\DD)}\\
  &+ \norm{g_L(\DD)-g_L(\DD')} \\
	&+ \norm{g_L(\DD')-g(\DD)},
\end{split}
\label{eq:Cayley1_244}
\end{equation}
For any $\e>0$, by Lemma \ref{lemma_functional_calculus1}, the first and the last terms of the right-hand side of \ref{eq:Cayley1_244} can be made smaller than $\e/2$ by choosing $L$ large enough. Moreover, for any $L\in\NN$,
\[\norm{g_L}_{\cC}=\sum_{l=0}^{L}l\abs{c_l}\leq \sum_{l=0}^{\infty}l\abs{c_l}=\norm{g}_{\cC},\]
so, by Theorem \ref{theo:Cayley1} for finite Cayley expansions, for every $\e>0$ 
\begin{equation}
\norm{g(\DD)-g(\DD')} \leq \norm{g}_{\cC} \Big((\norm{\DD}+1)\frac{\norm{\mathbf{E}}}{1-\norm{\mathbf{E}}} +\norm{\mathbf{E}} \Big) +\e.
\label{eq:Cayley1_2rrt5}
\end{equation}
Since (\ref{eq:Cayley1_2rrt5}) is true for every $\e>0$, we must have (\ref{eq:Cayley1}).
\end{proof}

\bibliographystyle{IEEEtran}
\bibliography{IEEEabrv,IEEE_Transferability}

\end{document}